\newif\ifarxiv
\arxivtrue   

\documentclass{article}
\ifarxiv
  \usepackage[preprint]{neurips_2026}
\else
  \usepackage{neurips_2026}
\fi

\usepackage[utf8]{inputenc}
\usepackage[T1]{fontenc}
\usepackage{microtype}
\usepackage{hyperref}
\usepackage{url}
\usepackage{booktabs}
\usepackage{amsmath,amsfonts,amssymb}
\usepackage{amsthm}
\usepackage{graphicx}
\usepackage{xcolor}
\usepackage{enumitem}

\newtheorem{theorem}{Theorem}
\newtheorem{corollary}[theorem]{Corollary}

\theoremstyle{definition}
\newtheorem{definition}[theorem]{Definition}
\theoremstyle{remark}

\hypersetup{draft}

\newif\ifshowtodos

\newif\ifcutversion
\cutversiontrue   

\long\def\cutblock#1\endcutblock{\ifcutversion\else#1\fi}

\newcommand{\ebone}{eb1}
\newcommand{\ebonepro}{eb1-pro}
\newcommand{\ebonepreview}{eb1-preview}
\newcommand{\ebonedelta}{eb1-delta-preview}
\newcommand{\ebonefrontier}{eb1-frontier-preview}
\newcommand{\gptoss}{GPT-OSS-20B}
\newcommand{\gptmini}{GPT-5-mini}
\newcommand{\gptfive}{GPT-5.4}
\newcommand{\opus}{Claude Opus~4.6}
\newcommand{\kimi}{Kimi-K2.5}
\newcommand{\gemini}{Gemini-2.5-Flash}
\newcommand{\qwen}{Qwen-3.5-9B}

\newcommand{\bon}{BoN}
\newcommand{\ada}{AdaEvolve}
\newcommand{\evox}{EvoX}

\newcommand{\A}{\mathcal{A}}

\newcommand{\gamble}{GAMBLe}

\title{Don't gamble, \gamble: An Analytical Framework for AI-Driven Research Systems}

\ifarxiv
\author{%
  Marquita Ellis\thanks{Corresponding author.} \\
  IBM Research\\
  \texttt{mme@berkeley.edu} \\
  \And
  Paul Castro \\
  IBM Research\\
}
\fi

\begin{document}

\maketitle


\begin{abstract}
AI-Driven Research Systems (ADRS)---systems coupling LLMs with automated evaluation to discover algorithms, proofs, and designs---are being optimized and adopted across domains, but the tools to analyze them have not kept pace. ADRS performance depends on component interactions that are poorly understood, expensive to explore, and (as we show) not well captured by standard convergence guarantees. These guarantees rely on structural assumptions that do not hold under the ADRS process we formalize.
We introduce \gamble{}, a framework that decomposes ADRS behavior into four parameters (generator $G$, assessor $\A$, discovery mechanism $\mathcal{M}$, budget $B$) and one compositional object, the \emph{effective landscape} $L_{\text{eff}} = \A \circ G$, which reveals that distinct generator--assessor pairs induce structurally different per-problem optimization landscapes.
We exercise the framework on $760{+}$ replicated runs ($>$46{,}000 iterations) spanning generators from single LLMs to dynamically-adaptive ensembles, mechanisms from greedy selection to co-evolutionary meta-search, and three NP-hard problems whose assessors range from continuous scoring to cliff functions.
The experiments reveal no total ordering of generators or mechanisms: frontier models can underperform open-source alternatives and the simplest mechanism sometimes outperforms state-of-the-art meta-search. Results show that even under limited budgets ($60$ iterations per run), the right component choices can improve performance by $13$--$67\%$ and search efficiency by $6$--$39\times$.
\end{abstract}

\section{Introduction}
\label{sec:intro}

Recent systems for automated scientific and algorithmic discovery (FunSearch \citep{Romera-Paredes2024}, AlphaEvolve \citep{AlphaEvolve2025}, LEVI \citep{tanveer2026levi}, and others) share a common architecture: a language model generates candidate solutions, a scoring function evaluates them, and a search algorithm directs the process by selecting parents, constructing prompts, and adapting strategy over time. We refer to these as \textbf{AI-Driven Research Systems} (ADRS), following \citet{cheng2025barbariansgateaiupending}. These systems have produced new mathematical lower bounds on the cap set problem, improved matrix multiplication algorithms, and competitive solutions to open optimization challenges. The architecture can work; the question is \emph{when and why} it works well across problems.

We identify the minimal decomposition needed to analyze ADRS behavior: the \textbf{generator} $G$ (the entire candidate-producing system), the \textbf{assessor} $\A$ (the evaluation system), and the \textbf{discovery mechanism} $\mathcal{M}$ (the search algorithm and its configuration), operating under a computational budget $B$ to explore a given problem landscape $L$. Recent engineering advances have expanded each component's design space: $G$ from single models to agentic multi-model systems \citep{hamadanian2026glia, qu2026bilevel}, $\A$ from scalar scoring to rich feedback \citep{agrawal2025gepa, cheng2026adrsdatabases}, and $\mathcal{M}$ from greedy to co-evolutionary meta-search \citep{liu2026evox, cemri2026adaevolve}. Exhaustive search over the resulting configuration space is quickly becoming intractable: each iteration consumes LLM inference (or, for compound generators, multiple coordinated calls), runs span tens to hundreds of iterations, and building confidence in near-optimality requires replication across runs (as we show), multiplying costs further. Without principled characterization, the compute, energy, token, and expert time costs of navigating this space grow with the complexity of the system.

ADRS lack theoretical foundations for efficient application and optimization. At least four sources of variability make it difficult to identify which component is limiting a given system: generator sensitivity, $G \times \mathcal{M}$ interaction, configuration sensitivity, and run-to-run variance. We formalize their origins (Section~\ref{sec:framework}) and expose each empirically (Section~\ref{sec:empirical}).
\cutblock
\begin{enumerate}
\item \textbf{Generator sensitivity.} Performance depends heavily on which model serves as the generator, in ways that do not track capability benchmarks or model scale. In our experiments, \opus{}---among the highest-ranked models on general benchmarks---scored $21.5$ while the smaller \gptmini{} scored $45.8$ on the same problem and mechanism.
\item \textbf{Component interaction.} The same mechanism can dramatically help one generator but not another, and the assessor's signal structure determines whether any mechanism can make progress at all. In our experiments, switching mechanism improved one model's score by $4\times$ while changing another's by less than $6\%$; the simplest mechanism (Best-of-N) outperformed both adaptive alternatives on some generators; and on one problem, no generator--mechanism combination produces any improvement because the assessor provides no gradient.
\item \textbf{Configuration sensitivity.} The same generator under different configurations (mechanism choice, prompt structure, hyperparameters) spans the full score range---in one case from $0$ to $77.4$---yet no ordering is preserved across generators or problems.
\item \textbf{Run-to-run variance.} Identical configurations produce different scores across runs, with coefficients of variation around $15\%$---large enough that single runs cannot reliably rank configurations.
\end{enumerate}
\endcutblock

\paragraph{Contributions.} We show that standard analytical tools' assumptions are often violated in ADRS, and introduce the minimal machinery necessary to begin bridging the gap:
\begin{enumerate}[leftmargin=*]
\item We formalize the ADRS process and prove that the best-score process $\{s^*_t\}$ is \textbf{not Markov} (Theorem~\ref{thm:nonmarkov}).
The full state (run history and mechanism state) \emph{is} Markov but evolves in a \textbf{growing-dimensional} space (the history gains an entry every step). While this process can be embedded in a fixed infinite-dimensional state space, many standard quantitative convergence analyses rely on structural assumptions---fixed representations, stationary objectives, stable transition operators---that are often violated in ADRS due to history-dependent context construction and adaptive mechanisms. Moreover, natural scalar summaries such as $s^*_t$ are not, in general, sufficient statistics for the process, so convergence behavior is not fully determined by low-dimensional progress metrics alone. Under history-preserving context construction, initial conditions propagate: runs starting from different seeds may follow persistently different trajectories even at the same $s^*_t$ (Appendix~\ref{app:persistence}).
\item We define the \textbf{effective landscape} $L_{\text{eff}} = \A \circ G$ and show that different generators can induce structurally different landscapes on the same problem (Theorem~\ref{thm:leff}); generator sensitivity in our data reflects structural differences rather than purely sampling noise. Ensemble generators can \textbf{escape barriers} that trap any single generator.
\item We define the \textbf{generator ceiling} $s^*_\infty(G, \A)$ and \textbf{system ceiling} $s^*_\infty(G, \A, \mathcal{M})$ and derive a \textbf{regime classification} ($G$-limited, $\A$-limited, $\mathcal{M}$-limited, saturated) that identifies the binding constraint with targeted evaluations rather than exhaustive ablation.
\item We validate the framework empirically across $760{+}$ replicated runs totaling ${>}46{,}000$ iterations, spanning 12~generators from 6~model families (static through dynamically adaptive), 3~search mechanisms, and 3~NP-hard problems whose assessors range from rich continuous scoring to cliff functions. The data reveals basin structure, $G \times \mathcal{M}$ interaction, and regime diversity including $\A$-limited configurations where no generator or mechanism can make progress (Figure~\ref{fig:p0-basins}; Section~\ref{sec:empirical}).
\end{enumerate}

\paragraph{Related work.}
$L_{\text{eff}}$ extends fitness landscape theory \citep{kauffman1993origins} to ADRS, where the generating operator is context-dependent, stochastic, and potentially adaptive (Appendix~\ref{app:ancestors}). Unlike AutoML \citep{feurer2015autosklearn}, our framework addresses non-additive component interactions. Our $(G, \A, \mathcal{M})$ decomposition is analytical rather than architectural; we map it to \citeauthor{cheng2025barbariansgateaiupending}'s five-component description in Appendix~\ref{app:cheng-mapping}. Several concurrent groups have observed generator sensitivity, $G \times \mathcal{M}$ interaction, and run-to-run variance empirically \citep{lehman2022evolutionlargemodels, cheng2025letbarbariansin, skydiscover2026, liu2026evox, cemri2026adaevolve, tanveer2026levi, agrawal2025gepa, cheng2026adrsdatabases, hamadanian2026glia, karimi2026engram, qu2026bilevel}; none provides a theoretical account of why these phenomena arise or how to diagnose the limiting factor with targeted evaluations rather than exhaustive ablation (Appendix~\ref{app:concurrent}).

\section{The \gamble{} framework}
\label{sec:framework}

We define the \gamble{} framework, beginning with a process model of how $G$, $\A$, and $\mathcal{M}$ interact within a run, then prove two structural results: non-Markov best-score dynamics, and generator-dependent effective landscapes. The \textbf{generator} $G$ produces candidates given context; the \textbf{assessor} maps candidates to scores, $\A: \mathcal{X} \to \mathbb{R}$, where $\mathcal{X}$ is the candidate space defined by the problem \textbf{landscape} $L$; and the \textbf{discovery mechanism} $\mathcal{M}$ directs exploration, encompassing parent selection, prompt construction, parameter adaptation, and potentially meta-level strategy evolution. Crucially, $\A$'s signal reaches $G$ only through $\mathcal{M}$'s context construction: $\mathcal{M}$ selects which scores, candidates, and history to surface, so what $G$ ``sees'' of the landscape depends on both $\A$'s fidelity and $\mathcal{M}$'s filtering. When $\A$ cannot distinguish candidates, no $\mathcal{M}$ can provide useful signal---$\A$ can be a binding constraint.
We formalize $\A$ as scalar-valued; structured feedback \citep{agrawal2025gepa} enters through context construction $C$ rather than extending $\A$'s codomain, so the theorems apply to $\{s^*_t\}$ regardless of feedback richness (feedback affects $L_{\text{eff}}$ geometry via $C$'s effectiveness). Notation is summarized in Appendix~\ref{app:notation}; all assumptions are stated formally in Appendix~\ref{app:assumptions}; alternative formulations considered are in Appendix~\ref{app:alternatives}.

Every ADRS produces a \textbf{run history} $D_t$: the append-only record of all data produced through step $t$. $\mathcal{M}$ accesses the history through context construction $C$, which determines what subset of $D_t$ the generator sees. Without history-dependent context construction, the process degenerates to i.i.d.\ sampling and gains come only from drawing more samples, not from learning to generate better candidates.

\begin{definition}[ADRS trajectory]
\label{def:trajectory}
An ADRS run produces states $(D_0, \mathcal{M}_0) \to \cdots \to (D_T, \mathcal{M}_T)$, where $D_t = \{(x_i, s_i)\}_{i=1}^{t}$ is the run history (projected to candidate-score pairs) and $\mathcal{M}_t$ is the mechanism state at step $t$. \textbf{At each step}, the system \textbf{(1)}~constructs context $c_t = C(D_t, \mathcal{M}_t)$, \textbf{(2)}~generates a candidate $x_{t+1} \sim G(\cdot \mid c_t)$, \textbf{(3)}~evaluates $s_{t+1} = \A(x_{t+1})$, and \textbf{(4)}~records the result, $D_{t+1} = D_t \cup \{(x_{t+1}, s_{t+1})\}$, and updates mechanism state $\mathcal{M}_{t+1} = U(D_t, \mathcal{M}_t, x_{t+1}, s_{t+1})$. The best-score process is $s^*_t = \max_{i \leq t} s_i$.
\end{definition}

\subsection{The best-score process is not Markov}
\label{sec:nonmarkov}

\begin{theorem}[Non-reducibility]
\label{thm:nonmarkov}
The best-score process $\{s^*_t\}$ is \textbf{not Markov} for an ADRS satisfying:
\begin{enumerate}[leftmargin=2em, label=(A\arabic*)]
\item Faithful context construction: $\exists\, D_t \neq D_t'$ with the same best score ($s^*_t = \max_{i \leq t} s_i$) such that $C(D_t, \mathcal{M}_t) \neq C(D_t', \mathcal{M}_t)$,
\item Context-dependent generation: $c \neq c' \Rightarrow G(\cdot \mid c) \neq G(\cdot \mid c')$ for $c, c'$ in the range of $C$,
\item Upper-tail-separating assessor: for any threshold $w$ (in particular $w = s^*_t$) and distinct $G(\cdot \mid c) \neq G(\cdot \mid c')$ that both admit improvement above $w$, the distributions of $\max(w,\, \A(x))$ are not identical,
\end{enumerate}
\end{theorem}

\begin{proof}
\emph{State dependence.} Take $(D_t, \mathcal{M}_t)$, $(D_t', \mathcal{M}_t)$ with $\max D_t = \max D_t' = w < \sup \A(\mathcal{X})$ but $D_t \neq D_t'$. Faithfulness (A1) gives $C(D_t, \mathcal{M}_t) \neq C(D_t', \mathcal{M}_t)$; context-dependence (A2) gives distinct generation distributions $G(\cdot \mid c_t) \neq G(\cdot \mid c_t')$; upper-tail separation (A3) then gives distinct distributions of $s^*_{t+1} = \max(w, s_{t+1})$, so $s^*_t$ alone does not determine the distribution of $s^*_{t+1}$.

\emph{Non-Markov.} Two realizations can reach $s^*_2 = w > s^*_0$ via different paths: improving at step~1 (history $(s^*_0, w, w)$) or at step~2 (history $(s^*_0, s^*_0, w)$). Their histories $D_2$ differ. By state dependence, these produce distinct distributions of $s^*_3$. Both conditionals share $s^*_2 = w$ but prior history $s^*_1$ is informative, so $\{s^*_t\}$ is not Markov.
\end{proof}

A1--A3 hold by construction for any ADRS that shows prior candidates to the generator (A1), uses a context-sensitive generator (A2), and has a non-degenerate assessor (A3); see Appendix~\ref{app:assumptions}.
In practice, the run history may include structured assessor feedback \citep{agrawal2025gepa}, generation reasoning traces, and experimental logs \citep{karimi2026engram, hamadanian2026glia}. Richer histories make A1 easier to satisfy---more ways for $D_t$ to differ---so systems that persist and leverage such data strengthen the non-Markov property. Moreover, under history-preserving context construction, initial conditions can persist: when $C(D_t, \mathcal{M}_t)$ carries information from $D_0$, the generation distribution remains $D_0$-dependent for all $t$ (Appendix~\ref{app:persistence}).

\textbf{Consequence.} The full state $(D_t, \mathcal{M}_t)$ \emph{is} Markov---but it evolves in a \emph{growing-dimensional} space: the history gains an entry every step. Although the process can be embedded in a fixed infinite-dimensional state space, many existing rate guarantees in optimization, bandit theory, and evolutionary algorithms typically assume fixed representations, stationary objectives, or stable transition operators---assumptions often violated in ADRS through history-dependent context construction and adaptive mechanisms (Appendix~\ref{app:standard-tools}). Scalar progress metrics such as $s^*_t$ do not determine future behavior, so analysis requires tracking the full history. When the effective landscape contains regions with different improvement probabilities, both reachable from the initial history, independent replications can produce multimodal final-score distributions. Because $\{s^*_t\}$ is not Markov, the score alone does not determine which region a run is exploring---two runs at the same $s^*_t$ but with different histories can have different improvement probabilities (Theorem~\ref{thm:nonmarkov}). Neither proximity to the ceiling, nor the binding constraint, nor the structure of final-score distributions can be determined from a single trajectory---all require independent replications.

\paragraph{Relation to inference-time scaling laws.}
The non-Markov property is what separates ADRS from classical \emph{inference-time scaling}, where a fixed generator is sampled repeatedly and a best, majority-voted, or verifier-selected candidate is returned \citep{brown2024monkeys, chen2024morecalls}. That regime is the non-adaptive limit of our process, in which context construction is history-independent: the draws are i.i.d., and gains come from per-sample compute (e.g., longer chains of thought, which raise the per-sample success probability) or from drawing more samples, not from conditioning generation on the content of the run history (A1). This i.i.d.\ limit reduces to the familiar best-of-$N$ coverage law (Corollary~\ref{cor:its}, Appendix~\ref{app:its}).

\noindent Departing from this i.i.d.\ regime is graded: if $C$ varies with the current best score alone, the per-sample success probability becomes state-dependent yet $\{s^*_t\}$ remains Markov (an inhomogeneous coverage process); dependence on history \emph{content} (A1) is what makes the improvement probability history-dependent and the process non-Markov (Theorem~\ref{thm:nonmarkov}). The coverage ceiling of the i.i.d.\ regime (the highest score with positive per-sample probability, attained as the number of samples grows) is the generator ceiling $s^*_\infty(G, \A)$ restricted to the single context $c_0$; adaptive mechanisms exist to construct better contexts and exceed it. With an automatic verifier, coverage converts directly into performance \citep{brown2024monkeys}; without one, selecting the best candidate is a separate \emph{identifiability} problem (the coverage--identifiability separation of \citealp{sunkaraneni2026boosting}), governed by the assessor $\A$.

\subsection{The effective landscape}
\label{sec:leff}

The system does not explore $L$ directly---it explores $L$ through the lens of what $G$ can generate. This gives the non-Markov result geometric content: generator sensitivity is not noise but reflects structurally different optimization landscapes, and which region of $L_{\text{eff}}$ a run explores is unobservable from $s^*_t$ alone. To formalize this:

\begin{definition}[Effective landscape]
\label{def:leff}
The \textbf{effective landscape} is $L_{\text{eff}}(G, \A) = \A_* \circ G$, where $G$ is the full family of conditional distributions $G(\cdot \mid c)$ over candidates, indexed by context $c$. $L_{\text{eff}}$ maps each context $c$ to the pushforward distribution $\A_* G(\cdot \mid c)$ over scores.
\end{definition}

When $G$ is static during a run---the common case, covering single models, static ensembles, and $\mathcal{M}$-controlled routing---$L_{\text{eff}}$ is time-invariant: a fixed surface the system navigates as $c_t = C(D_t, \mathcal{M}_t)$ evolves. When $G$ itself adapts (Section~\ref{sec:adaptive-g}), the landscape becomes time-varying: $L_{\text{eff},t} = \A_* \circ G_t$. All structural results below hold in both cases; for adaptive $G$, they apply at each instant $t$.

\begin{theorem}[Generator-dependent effective landscape]
\label{thm:leff}
Let $\A$ be a non-constant assessor (A4).
Two generators $G_1 \neq G_2$ can induce different effective landscapes: $L_{\text{eff}}(G_1, \A) \neq L_{\text{eff}}(G_2, \A)$.
\end{theorem}

\begin{proof}
Since $\A$ is not constant, there exist candidates $x, y$ with $\A(x) \neq \A(y)$. Generators that concentrate on $x$ vs.\ $y$ at some context $c$ then produce distinct pushforward score distributions.
\end{proof}

Note, while Theorem~\ref{thm:nonmarkov} requires A3 (upper-tail separation), Theorem~\ref{thm:leff} does not. Theorem~\ref{thm:nonmarkov} shows that landscape differences propagate to observable score dynamics, so it must exclude generators that differ only below the current $s^*_t$ (which would produce identical next-best-score distributions). Theorem~\ref{thm:leff} is only a structural claim about the landscape as an object, not about dynamics at a particular state, and requires only that $\A$ distinguishes some candidates (A4). Conversely, when $\A$ is coarse (many candidates share scores), $L_{\text{eff}}$ differences between generators collapse---the $\A$-limited regime (Table~\ref{tab:regimes}), where $\A$'s lack of discrimination renders generator choice irrelevant.

\textbf{Ensemble effective landscape.}
\label{cor:ensemble}
When $G = \sum_{i=1}^{k} w_i G_i$ is a flat mixture (fixed routing weights), the score distribution decomposes by linearity of pushforward: $\A_* G(\cdot \mid c) = \sum_{i=1}^{k} w_i \cdot \A_* G_i(\cdot \mid c)$.
A generator faces a \textbf{barrier} at context $c$ when its improvement probability vanishes: $\A_* G(\cdot \mid c)((s^*, \infty)) = 0$. The landscape may contain higher scores, but $G$ cannot reach them from $c$.
An ensemble's reachable set is the union of each component's, so ensembles can escape barriers that trap single generators. Ensemble benefit depends on $L_{\text{eff}}$ diversity, not ensemble size alone---adding $G_{k+1}$ dilutes existing weights ($\sum w_i = 1$), and helps only if it reaches score regions not already covered; consistent with \citet{cheng2025letbarbariansin}'s finding that ensembles beyond two models provided no additional benefit.
For verifying or routed generators (where component selection depends on candidate quality), the linearity decomposition does not necessarily hold. Our experiments suggest barrier escape extends to these systems, though the internal mechanism differs from flat-mixture union (Section~\ref{sec:empirical-variance}).

\subsection{Ceilings and regime classification}
\label{sec:ceilings}

\begin{definition}[Ceilings]
\label{def:ceilings}
The \textbf{generator ceiling} is $s^*_\infty(G, \A) = \sup\{s : \exists\, c \text{ s.t.\ } \A_* G(\cdot \mid c)([s, \infty)) > 0\}$---the supremum over all contexts $c$, including those that only the best mechanism with unlimited budget could construct. It is the intrinsic capability of $G$ under $\A$: an upper bound that no choice of $\mathcal{M}$ or $B$ can exceed. The \textbf{system ceiling} $s^*_\infty(G, \A, \mathcal{M}) \leq s^*_\infty(G, \A)$ restricts to contexts constructable by a specific $\mathcal{M}$ from the evolving history.
\end{definition}

Neither ceiling is directly observable; in practice, proximity is inferred from cross-configuration comparisons (Section~\ref{sec:empirical}): if varying $\mathcal{M}$ does not change scores, the system is likely near the generator ceiling.
These ceilings, together with the assessor's signal structure, induce a regime classification (Table~\ref{tab:regimes}). In practice, a system may exhibit features of more than one regime simultaneously, but identifying even a single binding constraint narrows the space of useful interventions.
\begin{table}[t]
\centering
\small
\begin{tabular}{@{}lp{5.0cm}p{4.6cm}@{}}
\toprule
\textbf{Regime} & \textbf{Definition} & \textbf{Implication} \\
\midrule
Generator-limited & $s^*_\infty(G, \A) < s_{\text{target}}$ & No choice of $\mathcal{M}$ or $B$ can reach the target. Change $G$. \\
$\mathcal{M}$-limited & $s^*_\infty(G, \A) \geq s_{\text{target}}$ but $s^*_\infty(G, \A, \mathcal{M}) < s_{\text{target}}$ & $G$ can produce better candidates but $\mathcal{M}$ cannot construct the contexts to elicit them. Change $\mathcal{M}$. \\
Budget-limited & $s^*_\infty(G, \A, \mathcal{M}) \geq s_{\text{target}}$ but $B$ insufficient & System would improve with more iterations. Increase $B$. \\
$\A$-limited & $G$ produces quality-varying candidates but $\A$ maps them to the same score & Optimization signal is too sparse to guide search. Enrich $\A$'s feedback structure. \\
Saturated & $s^* \approx s^*_\infty(G, \A, \mathcal{M})$ & Near-optimal under current configuration. Further runs have diminishing returns. \\
\bottomrule
\end{tabular}
\caption{Regime classification. Each regime identifies a binding constraint and intervention. Empirical illustrations appear in Section~\ref{sec:empirical}.}
\label{tab:regimes}
\end{table}
\vspace{-2pt}
The $\A$-limited regime is qualitatively different: the limitation is prior to the $G$/$\mathcal{M}$/$B$ hierarchy---no mechanism can navigate a landscape where $\A$ provides no optimization signal. This can arise even when $\A$ is \emph{correct} (it scores valid solutions accurately) but its feedback structure is too coarse---e.g., cliff scoring that collapses all invalid candidates to zero regardless of proximity to validity. Every call in this regime produces zero usable signal; since generation is expensive (each candidate requires at least one LLM call), the waste compounds. Distinguishing ``$\A$ provides insufficient signal'' from ``$G$ is producing uniformly poor candidates'' requires inspecting the candidates themselves (Section~\ref{sec:results-p11}).


\subsection{Static vs.\ adaptive generators}
\label{sec:adaptive-g}
An adaptive generator ($G$) learns or changes its approach during the course of a run.
A static generator navigates $L_{\text{eff}}$; an adaptive $G_t$ simultaneously navigates and reshapes it. 
The structural results (Theorem~\ref{thm:leff}, ceilings, regimes) apply to adaptive $G$ at each instant $t$ without modification. 
The non-Markov property (Theorem~\ref{thm:nonmarkov}) is strengthened: the generator's internal state $\theta_t$ is an additional source of path-dependence beyond the history. 
Ceilings become time-varying---$s^*_\infty(G_t, \A)$ is itself a process whose trajectory characterizes how fast the generator is learning. 
Adaptation is not automatically beneficial: a generator that overfits to early candidates may narrow its reachable set, decreasing $s^*_\infty(G_t, \A)$. 
The framework provides the quantities ($L_{\text{eff},t}$, ceiling trajectory) needed to measure whether adaptation helps, without assuming that it does.

\section{Empirical validation}
\label{sec:empirical}

\label{sec:experimental-setup}

\textbf{Benchmark and problems.}
We select three NP-hard problems for diversity in landscape structure and assessor signal (Table~\ref{tab:problems}) from Frontier-CS, an open-source competitive programming benchmark~\citep{mang2026frontiercs}.

\begin{table}[ht]
\centering
\footnotesize
\begin{tabular}{@{}llll@{}}
\toprule
\textbf{Problem} & \textbf{Task} & \textbf{$L$ (landscape)} & \textbf{$\A$ (scoring)} \\
\midrule
P0 & Polyomino packing (70 cases) & No known ceiling; local optima & Continuous ($\propto$ density) \\
P1 & Bounded 2D knapsack (3 cases) & Near-optimal reference known & Normalized; saturable at~$100$ \\
P11 & Palindrome Ham.\ path (3 cases) & Conjunctive hard constraints & Cliff: $0$ unless all hold \\
\bottomrule
\end{tabular}
\caption{Problems used in experiments (all NP-hard).}
\label{tab:problems}
\end{table}
\vspace{-6pt}

\textbf{Generators.}
We use 12 generators spanning 3 architectural categories (Table~\ref{tab:generators}).
Static generators have a fixed $G$ throughout the run, each inducing a time-invariant $L_{\text{eff}}$ (Theorem~\ref{thm:leff}). Network-of-networks (NoN) \citep{Davis2024} combine multiple models with internal verifiers, composing reachable sets across $L_{\text{eff}}$ (Section~\ref{sec:leff}). Static NoN have fixed components and routing; adaptive NoN reshape their generation distributions across the run, making $L_{\text{eff},t}$ non-stationary (Section~\ref{sec:adaptive-g}). The eb1 variants are closed-source NoN systems \citep{Davis2024} not yet publicly released.

\begin{table}[ht]
\centering
\footnotesize
\begin{tabular}{@{}llll@{}}
\toprule
\textbf{Category} & \textbf{Generator} & \textbf{Type} & \textbf{$L_{\text{eff}}$} \\
\midrule
Static (single model) (7) & \opus{} \citep{anthropic2023claude} & Commercial & Stationary \\
 & \gemini{} \citep{google2023gemini} & Commercial & Stationary \\
 & \gptfive{} \citep{openai2018gpt} & Commercial & Stationary \\
 & \gptmini{} \citep{openai2018gpt} & Commercial & Stationary \\
 & \gptoss{} \citep{openai2018gpt} & Open-weight, MoE & Stationary \\
 & \kimi{} \citep{moonshot2026kimi} & Open-weight, MoE & Stationary \\
 & \qwen{}\textsuperscript{\dag} \citep{qwen2024} & Open-weight & Stationary \\
\midrule
Static NoN (2) & \ebone{}, \ebonepro{}\textsuperscript{\dag} & NoN, fixed & Stationary \\
\midrule
Adaptive NoN (3) & \ebonepreview{} & NoN, adaptive & Non-stationary \\
 & \ebonedelta{} & NoN, adaptive & Non-stationary \\
 & \ebonefrontier{} & NoN, adaptive & Non-stationary \\
\bottomrule
\end{tabular}
\caption{Generators used in experiments. \textsuperscript{\dag}Partial problem coverage due to availability/latency.}
\label{tab:generators}
\end{table}
\vspace{-6pt}

\textbf{Discovery mechanisms.}
In order to vary $\mathcal{M}$ and $G$, we use SkyDiscover \citep{skydiscover2026} with 3 mechanisms, from a greedy baseline to state-of-the-art adaptive search.
Best-of-N (BoN), our greedy baseline (Appendix~\ref{app:mechanism-details}) carries no adaptive state, but is not stateless: its context depends on the growing history (Theorem~\ref{thm:nonmarkov}), so generator differences under BoN reflect $L_{\text{eff}}$ directly.
AdaEvolve \citep{cemri2026adaevolve} is an adaptive multi-island evolutionary search with $33{+}$ tunable parameters in SkyDiscover, including stagnation detection and island migration.
EvoX \citep{liu2026evox} additionally adapts the search strategy itself across iterations via co-evolutionary meta-search.

\textbf{Budget and replication.}
Each run uses $B = 60$ iterations. Convergence guarantees have not yet been established for ADRS (Section~\ref{sec:framework}), so we use a fixed budget to normalize comparisons across configurations while keeping total cost tractable; runs that repeatedly reach saturation (e.g., P1 score~$100$) are terminated early. We target $\geq 5$ independent replications per configuration, with additional runs for multimodal distributions until each detected basin contains $\geq 3$ observations (Appendix~\ref{app:replication}).

\subsection{Results on polyomino packing (Problem~0)}
\label{sec:empirical-variance}

\begin{figure}[!t]
\centering
\includegraphics[width=\linewidth]{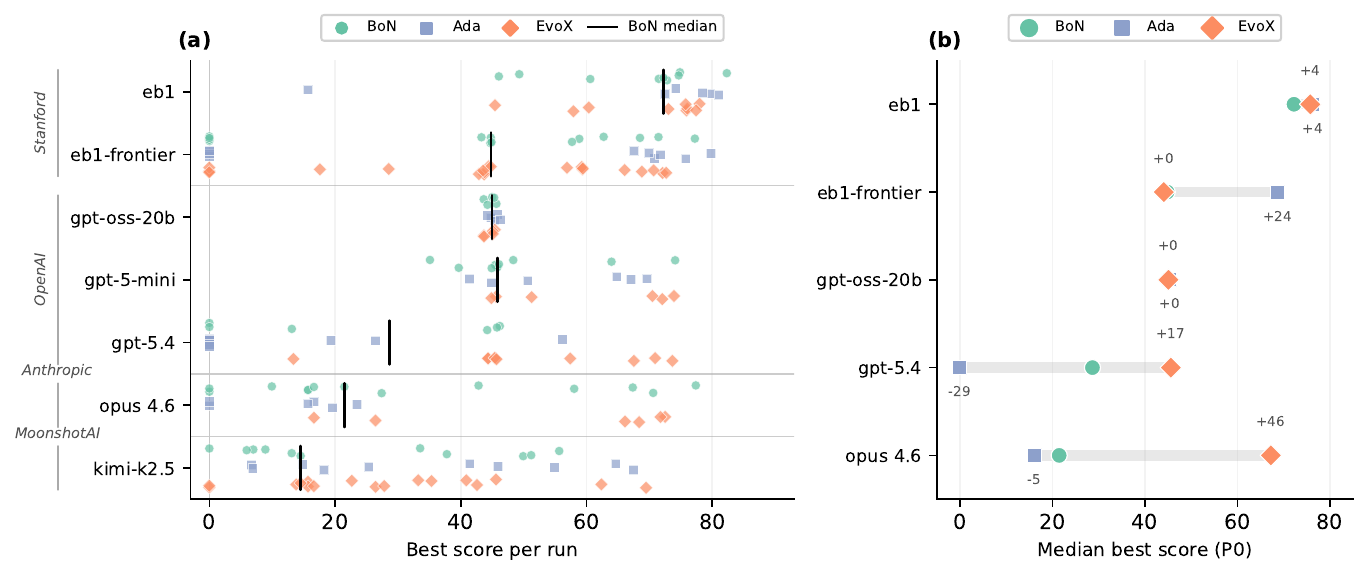}
\caption{Problem~0 results. (a)~Score distributions across generators and mechanisms. Each point is one run's best score; the black line marks the BoN median (greedy baseline). \ebonefrontier{} shown as representative of the dynamic eb1 variants (all three variants in Appendix Figure~\ref{fig:p0-basins-all}). \gemini{} and \qwen{} omitted (rare/no breakthroughs). (b)~Mechanism contribution relative to BoN median. Annotations show the score difference each adaptive mechanism achieves. The effect is non-monotonic: adaptive search can hurt (\opus{} \ada{}~$-5$, \gptfive{} \ada{}~$-29$). \gptmini{} and \kimi{} omitted (full version in Appendix Figure~\ref{fig:p0-m-contribution-all}).}
\label{fig:p0-basins}
\end{figure}
\vspace{-8pt}

Problem~0 is a polyomino packing instance: pack up to $10^4$ polyominoes of size 1--10 into a minimum-area axis-aligned rectangle (70~test cases, NP-hard). The assessor \citep{mang2026frontiercs} scores continuously, proportional to packing density.
Results across 12~generators from 6~model families and 3~mechanisms are summarized in Figure~\ref{fig:p0-basins} (representative subset; full set in Appendix Figure~\ref{fig:p0-basins-all}). We use BoN as a baseline because it isolates generator effects: with no adaptive state, score differences under BoN reflect $L_{\text{eff}}$ directly.

Generator scores span a wide range: the best (\ebone{}) reaches~82.3, while the weakest (\qwen{}) never exceeds~3.07, and \gemini{} rarely breaks through (4/20~runs), though its rare successes reach scores above~66 (Appendix Figure~\ref{fig:p0-basins-all}).
The ordering does not follow general capability rankings: \opus{} (BoN median~21.5) ranks below \gptmini{} (45.8) and \gptoss{} (45.0).
No generator exceeds a score of~82 regardless of mechanism. The \ebone{} family (Networks-of-Networks with internal verification, Table~\ref{tab:generators}) illustrates the complexity: \ebone{} base reaches~82 under \bon{} alone, yet variants adapting dynamically cluster near~44.
Whether a configuration ever exceeds score~0---an event we call \emph{breakthrough}---varies sharply by generator: \gptoss{} and \gptmini{} break through on every run, while \qwen{}, \gemini{}, and \gptfive{} frequently score~0.

For several generators, final scores cluster at discrete levels: most clearly \gptoss{} at~44 and the \ebone{} variants' shared~44 attractor (Figure~\ref{fig:p0-basins}); others show wider distributions consistent with multiple partially-resolved basins.
The \ebone{} variants (preview, frontier, delta) all cluster near~44 under \evox{}, and frontier/delta under \bon{} as well, a shared attractor across related but distinct generators and mechanisms suggesting this basin is a feature of the problem landscape.
Under AdaEvolve, \ebonepreview{} and \ebonefrontier{} shift to medians of 68--71, indicating that \ada{} can escape the~44 basin. The mechanism ranking reverses within the family: for \ebonepreview{}, EvoX \emph{loses}~$12$ points relative to BoN (Appendix Figure~\ref{fig:p0-m-contribution-all}), the opposite of the \ada{} benefit seen in \ebonefrontier{} ($+24$).
In contrast, \gptoss{} converges to~44 across all three mechanisms (CV~$\approx 2\%$), exhibiting a single accessible basin regardless of~$\mathcal{M}$.
\kimi{} shows the widest spread (0--69 under EvoX), consistent with multiple accessible basins.

Figure~\ref{fig:p0-basins}b quantifies each mechanism's contribution relative to BoN.
If mechanism adaptivity predicted performance, we would expect a consistent ordering BoN~$<$~AdaEvolve~$<$~EvoX; instead, no mechanism dominates across all generators.
For \ebone{}, both \ada{} and \evox{} improve on \bon{} by only~$+4$; mechanism choice is almost irrelevant.
For \opus{}, EvoX adds~$+46$ but AdaEvolve \emph{loses}~$5$ points relative to BoN: same generator, opposite mechanism effects.
\gptfive{} under AdaEvolve scores~0 on 8/11~runs (median~0 vs BoN median~28.7): guided search actively lowers performance for this pairing.
The breakthrough asymmetry reinforces this: \gptfive{} breaks through on 27\% of AdaEvolve runs but 100\% of EvoX runs---same $G$, $\A$, $L$, different $\mathcal{M}$, completely different breakthrough behavior.
In general, generators with lower BoN medians benefit more from guided search, but the relationship is non-monotonic and mechanism-specific.
No configuration reaches score~$100$ within $60$ iterations.

\subsection{Results on bounded knapsack (Problem~1)}
\label{sec:results-p1}

Problem~1 is a bounded 2D knapsack instance: select quantities of 12~item types to maximize total value subject to joint mass and volume constraints. Candidate solutions are evaluated on 3~test cases with continuous scoring relative to a known-optimal reference solution \citep{mang2026frontiercs}.
Most generator--mechanism combinations reach the optimum (score~$100$), but iterations-to-saturation and reliability vary significantly.

\begin{figure}[!t]
\centering
\includegraphics[width=\linewidth]{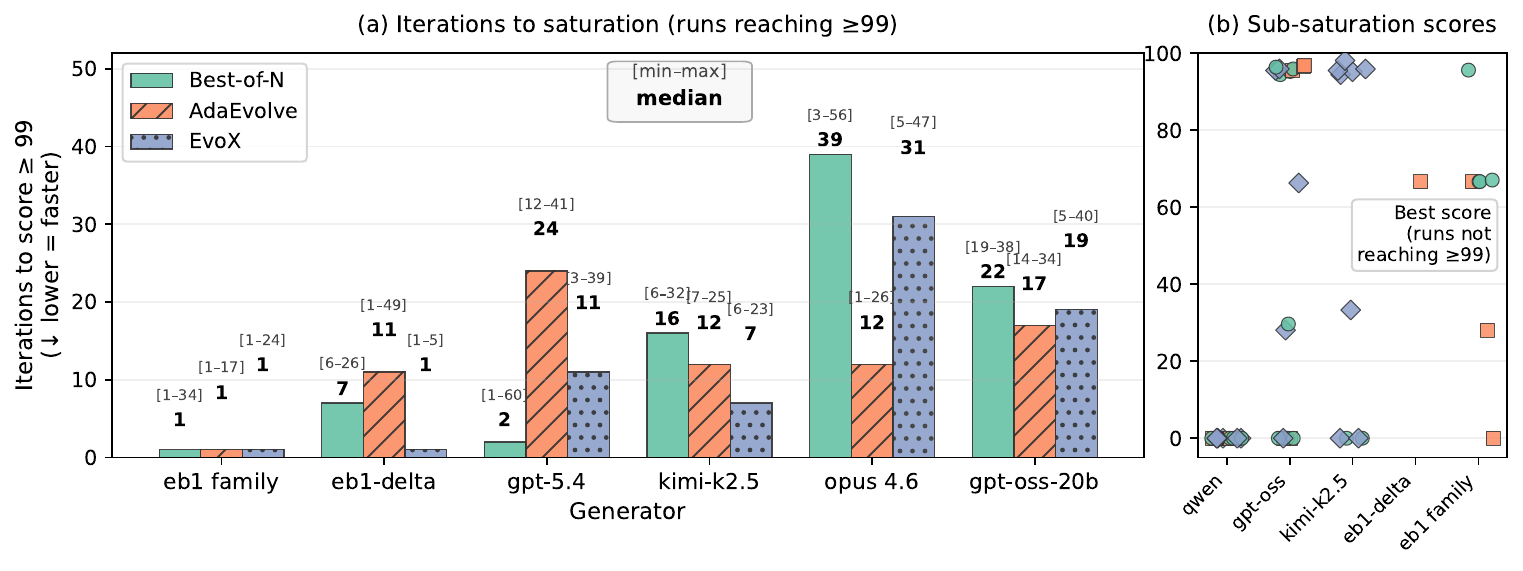}
\caption{Problem~1 search efficiency and reliability. (a)~Median iterations to saturation ($\text{score} \geq 99$) by generator and mechanism; brackets: $[\min, \max]$ over saturating runs. (b)~Best scores for non-saturating runs, colored by mechanism. The eb1 family groups eb1, \ebonepro{}, \ebonepreview{}, and \ebonefrontier{}. Full set including \gptmini{} and \gemini{} in Appendix Figure~\ref{fig:p1-iterations-all}.}
\label{fig:p1-iterations}
\end{figure}

Figure~\ref{fig:p1-iterations}a shows median iterations to saturation (score $\ge99$) across 6 representative generator classes and 3 mechanisms (full set including \gptmini{} and \gemini{} in Appendix Figure~\ref{fig:p1-iterations-all}).
The \textit{eb1 family} (\ebone{}, \ebonepro{}, \ebonepreview{}, and \ebonefrontier{}) are most efficient, all saturating at a median of $1$ iteration across all mechanisms.
Among static generators, the convergence rate spans nearly an order of magnitude: \kimi{} reaches saturation in $7$--$17$ iterations (median, depending on mechanism), while \gptoss{} ($17$--$22$) and \opus{} ($12$--$39$) form a slower tier.
\gptfive{} spans the full range ($2$--$24$) depending on mechanism, the widest spread among static generators.
This hierarchy is not predicted by general model capability rankings: \opus{} is the slowest static generator to saturate despite being one of the most capable on standard coding benchmarks at the time of result collection.

\label{sec:empirical-interaction}
The mechanism ranking is not consistent across generators.
\opus{} shows strong $\mathcal{M}$-sensitivity: AdaEvolve reaches saturation in a median of $12$ iterations versus $39$ for Best-of-N and $31$ for EvoX---a $3\times$ reduction in iterations from mechanism choice.
In contrast, \gptoss{} is $\mathcal{M}$-insensitive, saturating in $17$--$22$ iterations regardless of mechanism.
The mechanism ranking \emph{reverses} between generators: \ebonedelta{} saturates in $1$ iteration with EvoX but $11$ with AdaEvolve; \gptfive{} saturates in $2$ with Best-of-N but $24$ with AdaEvolve, the opposite of \opus{}'s ordering in both cases.
There is no universally best $\mathcal{M}$; the optimal pairing depends on the specific $G \times \mathcal{M}$ combination.

Among runs that do not reach saturation (Figure~\ref{fig:p1-iterations}b), scores cluster at $\approx 66.7$. Within the eb1 family, all sub-saturation runs come from the adaptive variants (preview, frontier, delta); static eb1 and eb1-pro always saturate, suggesting that adaptive $L_{\text{eff},t}$ reshaping occasionally steers into worse regions on this problem. Breakthrough rates vary by $G \times \mathcal{M}$: \gptoss{} achieves $69$--$92\%$ (depending on $\mathcal{M}$), \gemini{} $40$--$80\%$, \kimi{} on Best-of-N $71\%$, and \qwen{} $0/16$~runs even with the most adaptive mechanism.

\subsection{Results on palindrome path (Problem~11)}
\label{sec:results-p11}

Problem~11 requires finding a minimum-length palindromic move sequence that visits every blank cell of an $n \times m$ grid (up to $30 \times 30$, 3~test cases). Across ${>}16\text{K}$~iterations ($286$~runs, $22$~configurations spanning 10~generators from Table~\ref{tab:generators} and all 3~mechanisms), every run scores~0 (Figure~\ref{fig:p1-p11-matrix}).
This is the cleanest regime example in our study: unlike P0 (where breakthrough depends on $G \times \mathcal{M}$) or P1 (where breakthrough is stochastic but common), P11 exhibits universal failure regardless of configuration.

The universal failures do not reflect generators ignoring the problem. Across all runs, ${>}120\text{K}$ candidate programs were generated and evaluated. These candidates are substantive (median 220~lines, 93\% implementing BFS/DFS-based graph traversal) and often demonstrate correct mathematical reasoning about the palindrome constraint, but no candidate satisfies all constraints simultaneously.
The assessor \citep{mang2026frontiercs} scores via $\text{Clamp}(1 - (\ell - \ell^*)/\ell^*, \; 0, \; 1)$ where $\ell$ is the path length and $\ell^*$ a reference bound. However, this formula is only reached for \emph{valid} palindrome paths visiting all cells. Any candidate that fails validity scores~0 outright, with no partial credit for near-valid solutions.
This cliff structure means $\mathcal{M}$ receives no gradient signal: a candidate covering 90\% of cells scores identically to one covering 0\%. All ${>}120\text{K}$ candidates map to the same score, so no mechanism---regardless of sophistication---can distinguish improving from worsening candidates.
Formally, P11 satisfies Theorems~\ref{thm:nonmarkov} and~\ref{thm:leff} vacuously: no candidate scores above~0, so A3's antecedent is never met and the theorems make no predictions.
However, the framework's regime classification still provides actionable diagnosis independent of these guarantees. P11 is optimization-signal-limited under the current assessor: the assessor is \emph{correct} (it would score a valid solution accurately) but its cliff structure collapses all invalid candidates to a single score, flattening $L_{\text{eff}}$ so that no $\mathcal{M}$ can extract a gradient. This is diagnosed from the decomposition $L_{\text{eff}} = \A \circ G$: generators produce structurally relevant candidates (above), so the binding constraint is not $G$'s capability but $\A$'s feedback granularity. Enriching $\A$ to provide partial credit (e.g., for cell coverage or palindrome prefix length) is the necessary first step.

\section{Discussion}
\label{sec:discussion}

\begin{table}[!ht]
\centering
\footnotesize
\begin{tabular}{@{}lccp{5.5cm}@{}}
\toprule
\textbf{Aspect} & \textbf{Theory} & \textbf{Empirical} & \textbf{Evidence} \\
\midrule
Non-Markov property of $\{s^*_t\}$ & \checkmark & & Thm~\ref{thm:nonmarkov}; P0 replicated runs diverge from identical initial conditions (\S\ref{sec:empirical-variance}) \\
$L_{\text{eff}}$ depends on $G$ & \checkmark & & Thm~\ref{thm:leff}; P0 generator sensitivity (\S\ref{sec:empirical-variance}) \\
Ensemble barrier escape & \checkmark & \checkmark & \S\ref{sec:leff}; P0 eb1 variants (\S\ref{sec:empirical-variance}) \\
Regime taxonomy & \checkmark & & \S\ref{sec:ceilings}; P1 per-config (\S\ref{sec:results-p1}), P11 (\S\ref{sec:results-p11}) \\
Location of regime boundaries & & \checkmark & P1 breakthrough thresholds (\S\ref{sec:results-p1}) \\
Per-basin breakthrough rates & & \checkmark & P0 per-basin rates (\S\ref{sec:empirical-variance}) \\
Generator capability threshold & & \checkmark & P11 universal zero (\S\ref{sec:results-p11}) \\
Run-to-run variance & & \checkmark & P0 multimodal distributions (\S\ref{sec:empirical-variance}) \\
\bottomrule
\end{tabular}
\caption{Theory--empirics mapping with evidence.}
\label{tab:scope}
\end{table}
\vspace{-6pt}

The effective landscape unifies several empirical phenomena observed across our results and concurrent work.
\textbf{Generator sensitivity} (Sections~\ref{sec:empirical-variance}--\ref{sec:results-p1}): different generators \emph{can} induce structurally different $L_{\text{eff}}$ (Theorem~\ref{thm:leff}), and our data bear this out. Capability leaderboards do not predict ADRS performance; on P0, \gptoss{} (20B open-weight MoE) reaches median 45.0 versus 21.5 for \opus{}, and the ranking reverses across mechanisms. Model selection here behaves more like a topology question than a ranking question.
\textbf{$G \times \mathcal{M}$ interaction}: $\mathcal{M}$ navigates $L_{\text{eff}}(G, \A)$, so a mechanism suited to one landscape topology may be unsuited to another. This interaction is itself problem-dependent: \gptoss{} is $\mathcal{M}$-insensitive on P0 but $\mathcal{M}$-responsive on P1.
\textbf{Run-to-run variance}: can reflect basin structure when $L_{\text{eff}}$ has multiple accessible regions. Because $C$ samples from the history, same-region entries accumulate, producing self-reinforcing contexts, unless $\mathcal{M}$ actively injects diversity. Low variance can instead be an artifact of a coarse $\A$ that maps diverse candidates to the same score (cf.\ P11).
\textbf{Regime diversity}: even within a single problem, configurations occupy different regimes. On P1, \qwen{} is hard G-limited ($0/16$ breakthrough), \gptoss{} intermittently G-limited, and \ebone{} saturated.

Because these components are mutually conditioned (each component's contribution depends on what the others provide), effective improvement is best targeted by first diagnosing which component is binding rather than investing in any one by default.
The framework supports this diagnosis by connecting each design decision to a specific theoretical quantity ($L_{\text{eff}}$ geometry, ceilings, regime boundaries) that can be estimated with targeted evaluations rather than exhaustive ablation.
On P0, rich assessor signal enables $\mathcal{M}$ to exploit landscape structure; on P1, $\mathcal{M}$ affects efficiency but $G$ determines breakthrough; on P11, no $\mathcal{M}$ can help under the current cliff-scoring assessor, which collapses all invalid candidates to zero so no search strategy can extract a gradient.
The diagnosis points to action: enriching $\A$'s signal (e.g., partial credit for cell coverage or palindrome prefix length) could in principle restore optimization signal; whether $\mathcal{M}$ can then navigate, and what $G$ would then need to achieve, is untested.
This co-design perspective builds on independent component advances: mechanism diversity (co-evolutionary meta-search, island models) reveals $G \times \mathcal{M}$ interaction, and trajectory-aware mechanisms matter because the process is non-Markov: a mechanism that exploits history structure (distinct from a generator that adapts its own output distribution) can escape basins that memoryless mechanisms cannot \citep{karimi2026engram}.
Mechanism innovation, which dominates recent work, can therefore face diminishing returns precisely when $G$ or $\A$ is binding, as in both cases above.

Where mechanism innovation has diminishing returns, adaptive generation (internal verification, dynamic routing) and assessor enrichment may offer higher returns on many problems than further mechanism sophistication alone.
The assessor-enrichment side is already visible in public systems that route richer-than-scalar feedback to the generator \citep{agrawal2025gepa}; the adaptive-generation side is harder to study in fully open systems.
Our core findings rest on the 7 publicly available static generators, which already exhibit regime diversity, $G \times \mathcal{M}$ interaction, basin structure, and mechanism reversal; the (closed-source) eb1 family adds a suggestive data point on adaptive generation.
eb1 is $\mathcal{M}$-insensitive across problems, and eb1 base reaches higher basins (${\sim}82$) than its dynamic variants, which additionally reshape generation over a run (path-dependence beyond the history, Theorem~\ref{thm:nonmarkov}), consistent with adaptive generation reducing sensitivity to mechanism choice.
We cannot isolate whether internal verification, scale, or training drives this, and flag it only as a hypothesis for reproducible follow-up; no framework conclusion depends on eb1.

Diagnosis translates into concrete savings in iterations-to-solution.
Even on P1 (a ``solved'' problem), $G$-selection reduces iterations by ${\sim}39\times$ (1 for \ebone{} vs.\ 39 for \opus{}, both under \bon{}) and $\mathcal{M}$-selection by $3\times$ (12 for \ada{} vs.\ 39 for \bon{}, both with \opus{}).
Basin structure also reshapes how results should be summarized: when P0 scores cluster at $44$ and $72$, the mean ($58$) describes no actual run, which reframes the practitioner's question from ``how many runs to estimate a mean?'' to ``how many runs to resolve the mixture?''
The framework applies by construction across $(G, \A, \mathcal{M}, B, L)$: theory establishes the structures to diagnose (regimes, ceilings); empirics reveal basin structure and which regime a specific configuration falls into (Table~\ref{tab:scope}).
A sharper open question concerns the leverage of the earliest choices: because the process is non-Markov, $D_0$ content (not just $s^*_0$) can shape the trajectory (Theorem~\ref{thm:nonmarkov}), and the multimodal final-score distributions we observe from identical configurations (Section~\ref{sec:empirical-variance}) suggest that variation early in a run can commit it to a basin.
If so, inexpensive interventions on the initial history (seed selection, warm-starting) could rival the choice of generator or mechanism in impact, a possibility the non-Markov view makes precise and our replications motivate testing.

\paragraph{Limitations.}
The empirical validation uses a competitive programming benchmark; the theoretical results are architecture-general, but the specific phenomena could in principle be domain-specific. Independent results across math, systems, and science domains suggest they are not, though direct validation on additional domains would strengthen the case.
Search efficiency (iterations-to-saturation) is not computational efficiency: NoN generators perform opaque internal computation per API call, so one iteration for a verified generator may involve much more compute than for a single model. Our comparisons measure mechanism-level attempts needed, not cost per attempt.
Regime classifications are inferred from observed score distributions under finite replication; low-probability events (e.g., rare basin access or infrequent breakthrough) may remain undetected at the sample sizes used. The empirical results establish existence of the reported phenomena, not exhaustive characterization of the underlying landscapes.
Characterizing when $L_{\text{eff}}$ admits multiple separated score regions---contexts $c_1, c_2$ where $\A_* G(\cdot \mid c_i)$ concentrate on different score ranges with no improvement-probability path between them---is a topological question about the effective landscape that remains open.

\section{Conclusion}
\label{sec:conclusion}

We prove that the ADRS best-score process is non-Markov and that its full state is growing-dimensional (Theorem~\ref{thm:nonmarkov}); convergence guarantees based on fixed-dimensional Markov state do not directly apply. \gamble{} takes the first steps in bridging this gap: its effective landscape $L_{\text{eff}} = \A \circ G$ explains why generator sensitivity and $G \times \mathcal{M}$ interaction arise structurally, and its regime classification identifies the binding constraint ($G$, $\A$, $\mathcal{M}$, or $B$) with targeted evaluations rather than exhaustive ablation. Validation across $760{+}$ runs (${>}46{,}000$ iterations), 12~generators from 6~model families, 3~mechanisms, and 3~NP-hard problems exposes basin structure, interaction effects, and regime diversity.

The framework makes ADRS optimization systematic: each design decision maps to a specific theoretical quantity ($L_{\text{eff}}$ geometry, ceilings, regime boundaries) that researchers can estimate empirically. Because components are mutually conditioned---$\mathcal{M}$'s value depends on the $L_{\text{eff}}$ that $G$ and $\A$ jointly determine---effective improvement requires diagnosing which component is binding before investing in any one. When $\A$ is binding, no amount of $G$ or $\mathcal{M}$ spend helps; when a configuration is saturated, further iterations are wasted. As ADRS scale to longer runs and broader deployment, the non-Markov characterization and effective landscape geometry established here lay the groundwork for making ADRS efficient, supporting future work on autotuning, formal diagnostics, and scaling under resource constraints.

\section*{Acknowledgments}

\ifarxiv
We thank Jared Quincy Davis for early access to eb1, insightful feedback, and the whole eb1 model development team for their ongoing support.
\fi

\bibliography{paper1}
\bibliographystyle{plainnat}


\appendix
\section*{Appendix}

\noindent\textbf{Outline.}
\vspace{-4pt}
\begin{itemize}\setlength{\itemsep}{0pt}\setlength{\parskip}{0pt}
\item[\S A.] Notation reference
\item[\S B.] Assumptions
\item[\S C.] Why standard analytical tools do not apply
\item[\S D.] Alternative formulations considered
\item[\S E.] Relationship to Cheng et al.'s ADRS decomposition
\item[\S F.] Generation-verification structure
\item[\S G.] Methodological ancestors
\item[\S H.] Mechanism implementation details
\item[\S I.] Replication design
\item[\S J.] Persistence of initial conditions
\item[\S K.] Supplementary figures
\item[\S L.] Extended comparison with concurrent work
\end{itemize}

\section{Notation reference}
\label{app:notation}

\begin{table}[h]
\centering
\begin{tabular}{ll}
\toprule
\textbf{Symbol} & \textbf{Meaning} \\
\midrule
$G$ & Generator (single model, ensemble, or compound system) \\
$G(\cdot \mid c)$ & Generation distribution conditioned on context $c$ \\
$\A$ & Assessor, $\A: \mathcal{X} \to \mathbb{R}$ \\
$s = \A(x)$ & Score of candidate $x$ \\
$\mathcal{M}$ & Discovery mechanism + configuration \\
$L$ & Problem landscape (candidate space, structure, constraints) \\
$\mathcal{X}$ & Candidate space \\
$D_t$ & Run history through step $t$ (theorems use $(x_i, s_i)$ projection) \\
$\mathcal{M}_t$ & Mechanism state at step $t$ \\
$C$ & Context construction: $c_t = C(D_t, \mathcal{M}_t)$ \\
$U$ & Update function \\
$s^*_t$ & Best score at step $t$: $\max_{i \leq t} s_i$ \\
$L_{\text{eff}}(G, \A)$ & Effective landscape: $\A \circ G$ \\
$\A_*$ & Pushforward: $\A_* P$ is the score distribution induced by $\A$ on samples from $P$ \\
$s^*_\infty(G, \A)$ & Generator ceiling \\
$s^*_\infty(G, \A, \mathcal{M})$ & System ceiling \\
$B$ & Computational budget \\
\bottomrule
\end{tabular}
\caption{Notation reference.}
\label{tab:notation}
\end{table}

\section{Assumptions}
\label{app:assumptions}

The theoretical results in this paper depend on the following assumptions. We state them explicitly for auditability; each theorem references which assumptions it requires.

\noindent\textbf{A1. Faithful context construction.} Context construction uses history content beyond the current best score: there exist $D_t \neq D_t'$ with the same best score ($\max_{i \leq t} s_i = \max_{i \leq t} s_i'$) such that $C(D_t, \mathcal{M}_t) \neq C(D_t', \mathcal{M}_t)$ with positive probability. This excludes mechanisms that condition only on $s^*_t$ or ignore the history entirely. \emph{Used in:} Theorem~\ref{thm:nonmarkov}; persistence remark (Appendix~\ref{app:persistence}).

\noindent\textbf{A2. Context-dependent generator.} The generator's output distribution depends on the context: $c \neq c' \Rightarrow G(\cdot \mid c) \neq G(\cdot \mid c')$ for all $c, c'$ in the range of $C$. For LLM-based generators receiving contexts from $C$, this is empirically natural: contexts differ in substantive content (code, scores, run history), not merely in paraphrasing or irrelevant tokens, so distinct contexts produce distinct output distributions. \emph{Used in:} Theorem~\ref{thm:nonmarkov}; persistence remark (Appendix~\ref{app:persistence}).

\noindent\textbf{A3. Upper-tail-separating assessor.} For any threshold $w$ and any distinct generation distributions $G(\cdot \mid c) \neq G(\cdot \mid c')$ that both admit improvement above $w$ (i.e., $P(\A(x) > w \mid c) > 0$ and $P(\A(x) > w \mid c') > 0$), the distributions of $\max(w,\, \A(x))$ under $G(\cdot \mid c)$ and $G(\cdot \mid c')$ are not identical.

This is the load-bearing assumption for Theorem~\ref{thm:nonmarkov}. It is strictly about the \emph{upper tail}: two generation distributions that differ only below $w$ produce identical distributions of $s^*_{t+1} = \max(w, s_{t+1})$, so below-threshold differences cannot violate the Markov property at best score $w$. A3 requires that when contexts differ and improvement is possible from both, the distributions of the next best score differ---either through different improvement probabilities, or through different conditional distributions above $w$. The quantifier is restricted to the generation family $\{G(\cdot \mid c)\}_c$; we do not require this for arbitrary distributions over candidates.

For LLM-based generators, A3 is empirically natural: a prompt containing high-scoring code produces a different probability of \emph{exceeding} the current best than a prompt containing low-scoring code---the generation distribution doesn't partition into ``above current best'' and ``below'' independently of context.

\emph{Used in:} Theorem~\ref{thm:nonmarkov} (for the non-Markov step). \emph{Vacuous on:} P11's cliff assessor (Section~\ref{sec:results-p11}), where improvement probability is zero for all contexts ($\A_*$ maps every generation distribution to $\delta_0$). A3's antecedent is never satisfied, so the theorem is silent---not because the assumption is violated, but because no context admits improvement for it to distinguish.

\noindent\textbf{A4. Non-constant assessor.} There exist candidates $x, y \in \mathcal{X}$ with $\A(x) \neq \A(y)$. \emph{Used in:} Theorem~\ref{thm:leff}. Strictly weaker than A3.

\noindent\textbf{A5. Budget sufficiency.} The computational budget $B$ is large enough for the mechanism to explore multiple history states. Required for the budget-limited vs.\ saturated regime distinction to be meaningful. \emph{Used in:} regime classification (Table~\ref{tab:regimes}).

\noindent\textbf{A6. Assessor faithfulness.} The assessor $\A$ faithfully represents the optimization objective---higher $\A$-scores correspond to genuinely better candidates. The framework optimizes $\A$; if $\A$ is a misleading proxy, the system converges to optima of the proxy. \emph{Used in:} all interpretive claims that connect scores to candidate quality.

\noindent Assumptions A1--A4 are structural and verifiable from the system's architecture. A3 targets the upper-tail distribution within the generation family and directly delivers what the non-Markov proof requires. A5 is a design assumption about experimental adequacy. A6 is an implicit modeling assumption common to all score-based optimization.

\section{Why standard analytical tools do not apply}
\label{app:standard-tools}

Many standard quantitative convergence tools rely on structural assumptions that are often violated in ADRS. Classical rate guarantees for gradient methods assume a fixed objective and stable update operator \citep{nesterov2018lectures}; classical regret bounds for bandits assume a fixed or well-structured action set \citep{lai1985asymptotically, slivkins2019introduction}; EA runtime analyses typically assume fixed representations with stationary mutation kernels \citep{neumann2010bioinspired}.

The growing history can be embedded in an infinite-dimensional product space where the full-state process is Markov---and indeed time-homogeneous if the mechanism state is included---but the resulting representation grows with $t$ and does not, in general, admit tractable forms of the structural properties (e.g., contraction, ergodicity, compactness) on which standard analyses rely \citep{levin2017markov}. History-dependent context construction and adaptive mechanisms mean that existing analytical tools do not directly yield nontrivial quantitative bounds on convergence rates in this setting. The fundamental difficulty is that natural scalar summaries such as $s^*_t$ are not, in general, sufficient statistics for the process (Theorem~\ref{thm:nonmarkov}): convergence behavior depends on the full history, not the current best score alone.

\section{Inference-time scaling as the non-adaptive limit}
\label{app:its}

\begin{corollary}[Inference-time scaling as the non-adaptive limit]
\label{cor:its}
If context construction is history-independent ($C(D_t, \mathcal{M}_t) = c_0$ for all $t$---a fixed context, which in particular violates A1), then $x_1, x_2, \dots$ are i.i.d.\ draws from $G(\cdot \mid c_0)$, the best-score process $\{s^*_t\}$ is Markov, and the probability of exceeding a score $w$ within $N$ samples is
\[
\Pr(s^*_N > w) = 1 - (1 - p_w)^N, \qquad p_w := \Pr\big(\A(x) > w \mid c_0\big),
\]
with $p_w$ the per-sample success probability.
\end{corollary}
\begin{proof}
History-independence makes each draw independent of the past, so $\{s^*_t\}$ is the running maximum of i.i.d.\ draws---Markov, with transition depending only on the current value; $s^*_N > w$ fails iff all $N$ draws score $\le w$.
\end{proof}

\section{Alternative formulations considered}
\label{app:alternatives}

\paragraph{True landscape $L^*$.} We considered introducing $L^*: \mathcal{X} \to \mathbb{R}$ representing actual candidate quality, distinct from $\A$. Rejected: for computationally hard problems $\A = L^*$; for engineering problems $L^*$ is ill-defined. $L^*$ does no analytical work---all results operate on $\A$. We note $\A$-faithfulness as an implicit assumption: the assessor $\A$ is taken as ground truth for what constitutes improvement.

\paragraph{Prompt construction as separate parameter.} Rejected: examination of four discovery mechanisms shows prompt construction is tightly coupled to mechanism architecture. Separating them creates a parameter that cannot vary independently in practice.

\paragraph{Assessor embedded in landscape.} Initial formulation used $(G, \mathcal{M}, L)$ with $\A \subset L$. Revised to $(G, \A, \mathcal{M}, L)$: $\A$ is a design choice, and separating it enables assessor-limited diagnosis.

\section{Relationship to Cheng et al.'s ADRS decomposition}
\label{app:cheng-mapping}

\citet{cheng2025barbariansgateaiupending} define five architectural components of ADRS: Prompt Generator, Solution Generator, Evaluator, Storage, and Solution Selector. Our parameterization maps as follows:

\begin{table}[h]
\centering
\small
\begin{tabular}{@{}llp{6.4cm}@{}}
\toprule
\textbf{Cheng et al.\ component} & \textbf{Our parameter} & \textbf{Notes} \\
\midrule
Solution Generator & $G$ & The language model or model system \\
Evaluator & $\A$ & The scoring/assessment function \\
Storage & $D_t$ (partially) & $D_t$ is the data available to store; what is retained is a design choice within $\mathcal{M}$ \\
Solution Selector & $\mathcal{M}$ (partially) & Parent selection is part of the mechanism \\
Prompt Generator & $\mathcal{M}$ + system design & Context construction $C(D_t, \mathcal{M}_t)$ \\
\bottomrule
\end{tabular}
\caption{Mapping between \citet{cheng2025barbariansgateaiupending}'s five-component decomposition and our parameterization.}
\label{tab:cheng-mapping}
\end{table}

The Prompt Generator is the component least cleanly captured by $(G, \A, \mathcal{M}, L)$. It controls: (a)~problem statement formatting (fixed per problem---part of $L$), (b)~which previous solutions to show (determined by Solution Selector---part of $\mathcal{M}$), and (c)~the prompt template and content: how selected information is assembled into text. We absorb (c) into $C(D_t, \mathcal{M}_t)$, treating it as part of $\mathcal{M}$, because in current ADRS implementations prompt construction is tightly coupled to the discovery mechanism: AdaEvolve's prompts include island-specific context, EvoX's prompts include strategy performance history, and GEPA's prompts include rejection feedback. Changing $\mathcal{M}$ already changes prompt generation.

\section{Generation-verification structure}
\label{app:genverif}

Generators may be compound AI systems whose internal composition determines the structure of $L_{\text{eff}}$. Table~\ref{tab:generator-hierarchy} shows the composition hierarchy; each level builds on the previous.

\begin{table}[h]
\centering
\begin{tabular}{llll}
\toprule
\textbf{Level} & \textbf{Structure} & \textbf{Example} & \textbf{Effect on $L_{\text{eff}}$} \\
\midrule
Single & $G$ & One LLM & Base landscape \\
Ensemble & $\sum w_i G_i$ & Two-model mix & Smoothed barriers \\
Verified & $V \circ (G_1, \ldots, G_K)$ & Verifier judging $K$ models & Smoothed + filtered \\
\bottomrule
\end{tabular}
\caption{Composition hierarchy for generators.}
\label{tab:generator-hierarchy}
\end{table}

Ensemble composition smooths $L_{\text{eff}}$ by mixing reachable sets (Section~\ref{sec:leff}). Verified composition additionally filters: internal verifiers preferentially pass high-quality candidates before they reach the outer assessor $\A$. In general, composition nests recursively---generators contain verifiers that contain generators---producing arbitrarily deep filtering hierarchies \citep{Davis2024}. A verified generator may be \emph{static} (non-adaptive components and routing) or \emph{adaptive} (components or routing evolve across iterations, with or without a learning conductor). Our experiments span both: eb1 base is a static NoN, while eb1-preview, eb1-frontier-preview, and eb1-delta-preview are adaptive NoN whose generation distributions shift over a run (Table~\ref{tab:generators}).

\section{Methodological ancestors}
\label{app:ancestors}

AutoML systems such as auto-sklearn \citep{feurer2015autosklearn} operate over large, structured configuration spaces whose practical tractability depends on limiting effective parameter interactions---surrogate models work best when effective dimensionality is low and high-order interactions are limited. ADRS components interact strongly: $G \times \mathcal{M}$ interactions are pervasive in our data, and the binding constraint shifts across problems, so which component to optimize depends on the full $(G, \A, \mathcal{M})$ configuration. The configuration space is not a numeric hypercube but a choice among qualitatively different systems.

Fitness landscape theory in evolutionary computation \citep{kauffman1993origins} is the closest intellectual ancestor. Kauffman's NK model takes a fixed local mutation neighborhood as given; later formalizations (e.g., Stadler's $\langle$space, neighborhood, fitness$\rangle$ definition) make landscape topology explicitly dependent on the mutation operator, rendering it solver-configuration-dependent. $L_{\text{eff}}$ extends this lineage: the ``operator'' is now a generative system comprising at least one LLM, whose output distribution is context-dependent, stochastic, and potentially adaptive, making the landscape both generator-dependent and, for adaptive generators, time-varying (Section~\ref{sec:adaptive-g}).

\section{Mechanism implementation details}
\label{app:mechanism-details}

All mechanisms are implemented in SkyDiscover \citep{skydiscover2026} (commit \texttt{48bf7ae}). All generation hyperparameters (temperature, max\_tokens, top-p) use SkyDiscover's defaults at this commit unless otherwise stated. We document the precise context construction for each mechanism, since the information shown to the generator determines whether the process satisfies the faithfulness condition of Theorem~\ref{thm:nonmarkov}.

\paragraph{Best-of-N (BoN).}
BoN uses greedy parent selection with history-dependent context:
\begin{enumerate}
\item \textbf{Parent selection}: the highest-scoring program in the history (ties broken by insertion order). The same parent is reused for $N{=}5$ consecutive iterations before re-selection ($N_{\text{reuse}}{=}5$).
\item \textbf{Context programs}: $k_{\text{ctx}}{=}4$ programs sampled uniformly from the history's top-$k_{\text{pool}}{=}10$ by score (excluding the parent). These are non-best programs---the generator sees alternatives, not just the current optimum.
\item \textbf{Previous attempts}: up to $k_{\text{prev}}{=}3$ recent programs (from the last $w{=}100$ iterations), sorted by score, shown with their metrics and whether they improved or regressed relative to their parent.
\item \textbf{Prompt}: the parent's code, score breakdown, and evaluator feedback; context programs' code and scores; previous attempts with outcomes; improvement direction hints (score trend, solution length).
\end{enumerate}
BoN's context construction $C$ is faithful: two histories sharing the same best score but differing in non-maximal entries will produce different context samples and different previous-attempt histories, satisfying the conditions of Theorem~\ref{thm:nonmarkov}. Score ties are common in practice (e.g., $53$ programs sharing score $63.78$ within a single P0 run), so the ``strictly injective assessor'' assumption that would make BoN Markov-reducible does not hold.

BoN carries no adaptive state beyond the history: no population partitions, no stagnation counters, no strategy parameters. This makes it the minimal faithful mechanism---any simpler $C$ (e.g., showing only the best program with no context) would be unfaithful, reducing the process to i.i.d.\ sampling.

\paragraph{AdaEvolve.}
AdaEvolve \citep{cemri2026adaevolve} adds multi-island population structure with $33{+}$ tunable parameters. Its context construction includes island-specific parent selection, stagnation detection that triggers ``meta-guidance'' strategy shifts, and island migration. The mechanism state $\mathcal{M}_t$ includes per-island stagnation counters, migration history, and adapted strategy parameters.

\paragraph{EvoX.}
EvoX \citep{liu2026evox} adapts the search strategy itself via co-evolutionary meta-search. Its context construction includes strategy performance history across iterations and a switch interval ($20$ iterations in our experiments) that triggers strategy re-evaluation. The mechanism state includes the strategy archive and performance statistics.

\section{Replication design}
\label{app:replication}

Every configuration receives $\geq 5$ independent runs. For configurations whose sorted scores show gaps $> 15$ points (indicating distinct performance basins), we add runs until each detected cluster contains $\geq 3$ independent observations.
\section{Persistence of initial conditions}
\label{app:persistence}

\noindent\textbf{Remark (persistence of initial conditions).} Consider an ADRS with faithful $C$ and context-dependent $G$ under \emph{$D_0$-preserving} context construction: $C(D_t, \mathcal{M}_t)$ carries information from $D_0$ for all $t$. Then the generation distribution $G(\cdot \mid c_t)$ depends on $D_0$ for all $t \leq T$. The argument is a one-step induction: $D_0$ determines $c_0 = C(D_0, \mathcal{M}_0)$, so $G(\cdot \mid c_0)$ depends on $D_0$ by A2; and $D_0$-preservation keeps $D_0$ information in $c_t$, which A2 carries into $G(\cdot \mid c_t)$. This is an immediate consequence of context-dependence rather than a separate result; its content is the conditional, isolating \emph{when} initial conditions persist.

The hypothesis is also where it can fail. $D_0$-preservation holds only so long as $D_0$ remains competitive, or $\mathcal{M}$ or $G$ continue surfacing it in context. When the run improves past $D_0$, mechanisms typically replace early entries with higher-scoring candidates, weakening $D_0$'s influence on the context.

\section{Supplementary figures}
\label{app:supplementary-figures}

\begin{figure*}[h]
\centering
\includegraphics[width=\linewidth]{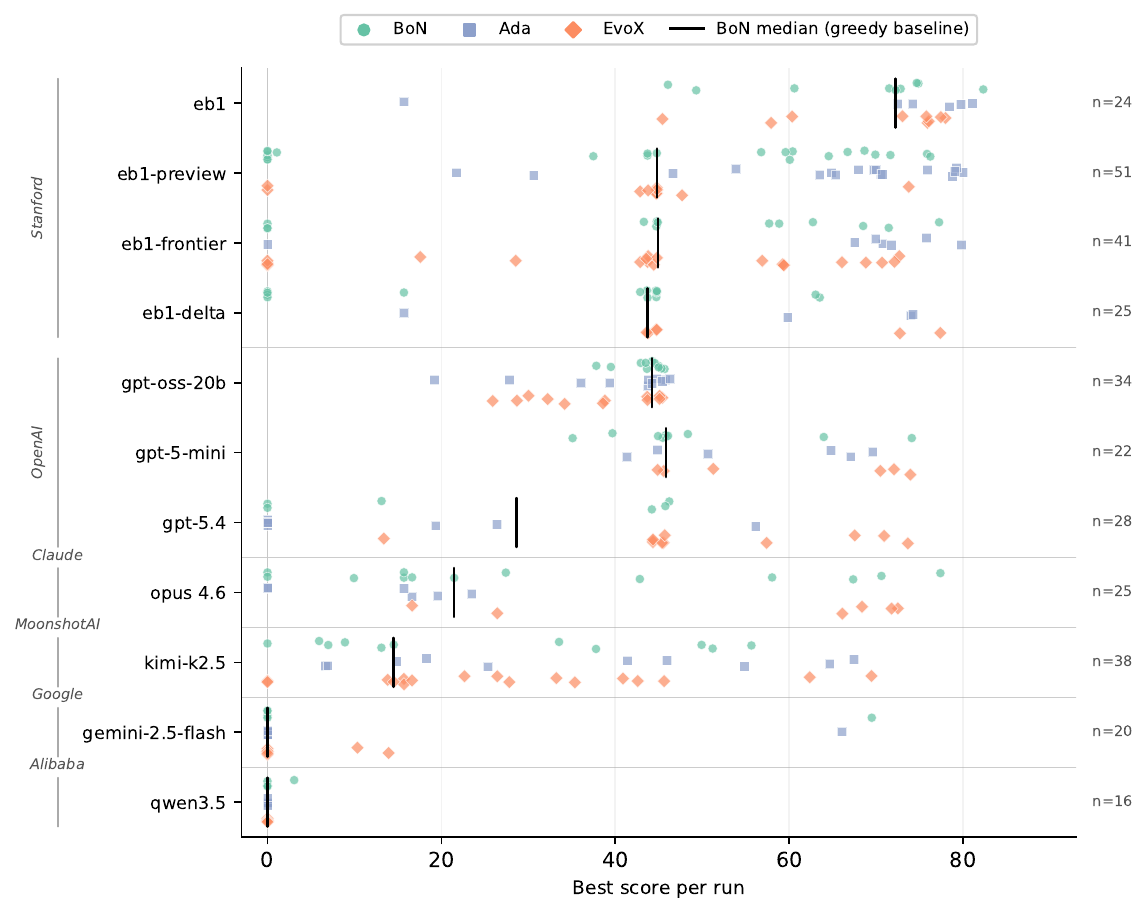}
\caption{P0 score distributions for all 12 generators individually (Figure~\ref{fig:p0-basins}a shows \ebonefrontier{} as representative). All configurations have $n \geq 5$ runs per mechanism; see Appendix~\ref{app:replication} for replication criteria.}
\label{fig:p0-basins-all}
\end{figure*}

\begin{figure}[h]
\centering
\includegraphics[width=0.75\linewidth]{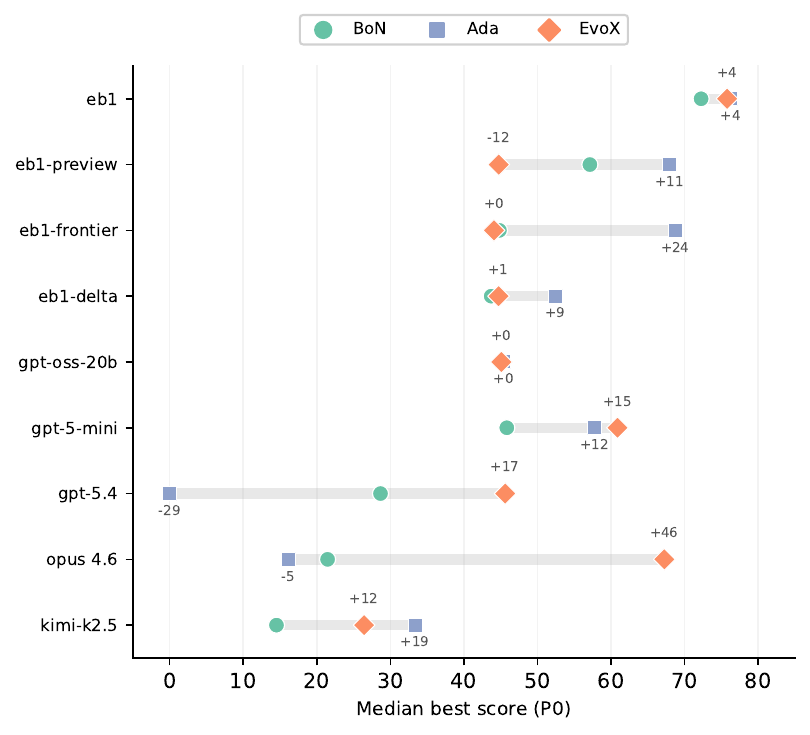}
\caption{Mechanism contribution relative to BoN median on P0 (full version of Figure~\ref{fig:p0-basins}b, with all eb1 variants shown individually). \gemini{} and \qwen{} omitted (median score~0 across all mechanisms).}
\label{fig:p0-m-contribution-all}
\end{figure}

\begin{figure*}[h]
\centering
\includegraphics[width=\linewidth]{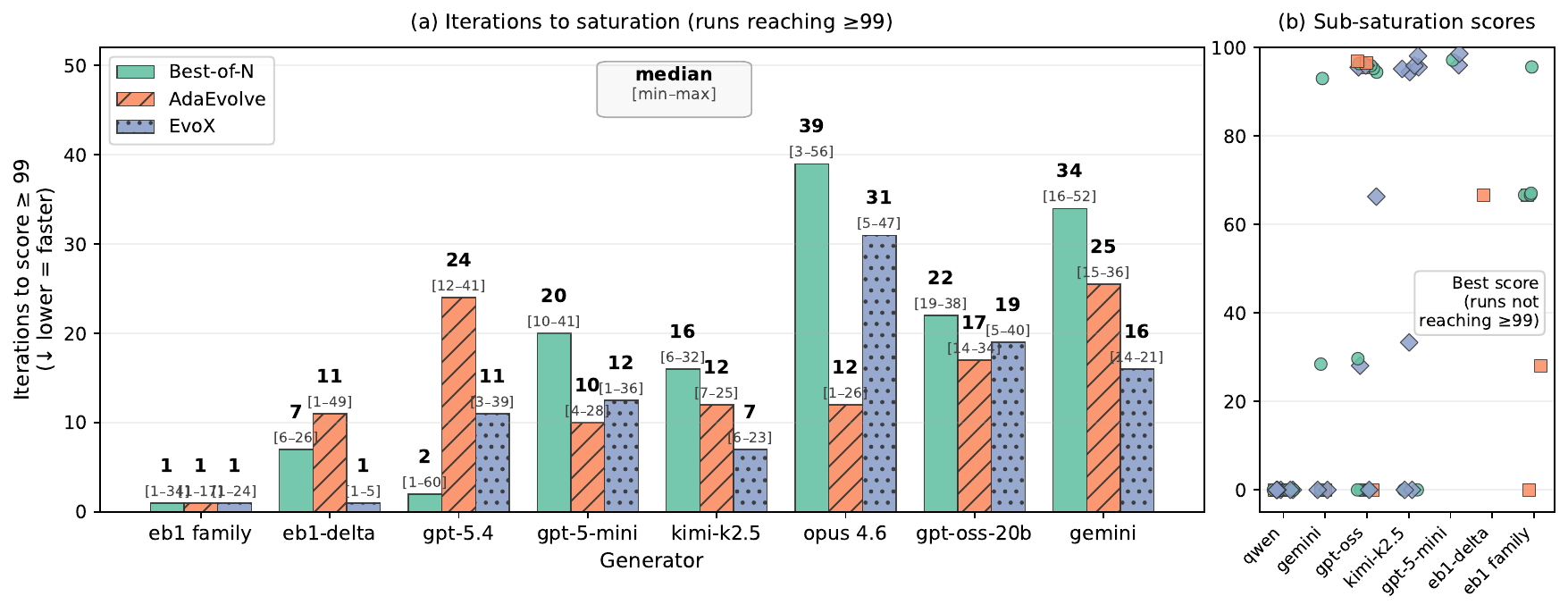}
\caption{P1 iterations to saturation and sub-saturation scores for all generators (full version of Figure~\ref{fig:p1-iterations}). Includes \gptmini{} and \gemini{}, omitted from the main figure for space. \gptmini{} behaves similarly to \kimi{} (fast saturation); \gemini{} is in the slower tier with \gptoss{} and \opus{}.}
\label{fig:p1-iterations-all}
\end{figure*}

\begin{figure*}[h]
\centering
\includegraphics[width=\linewidth]{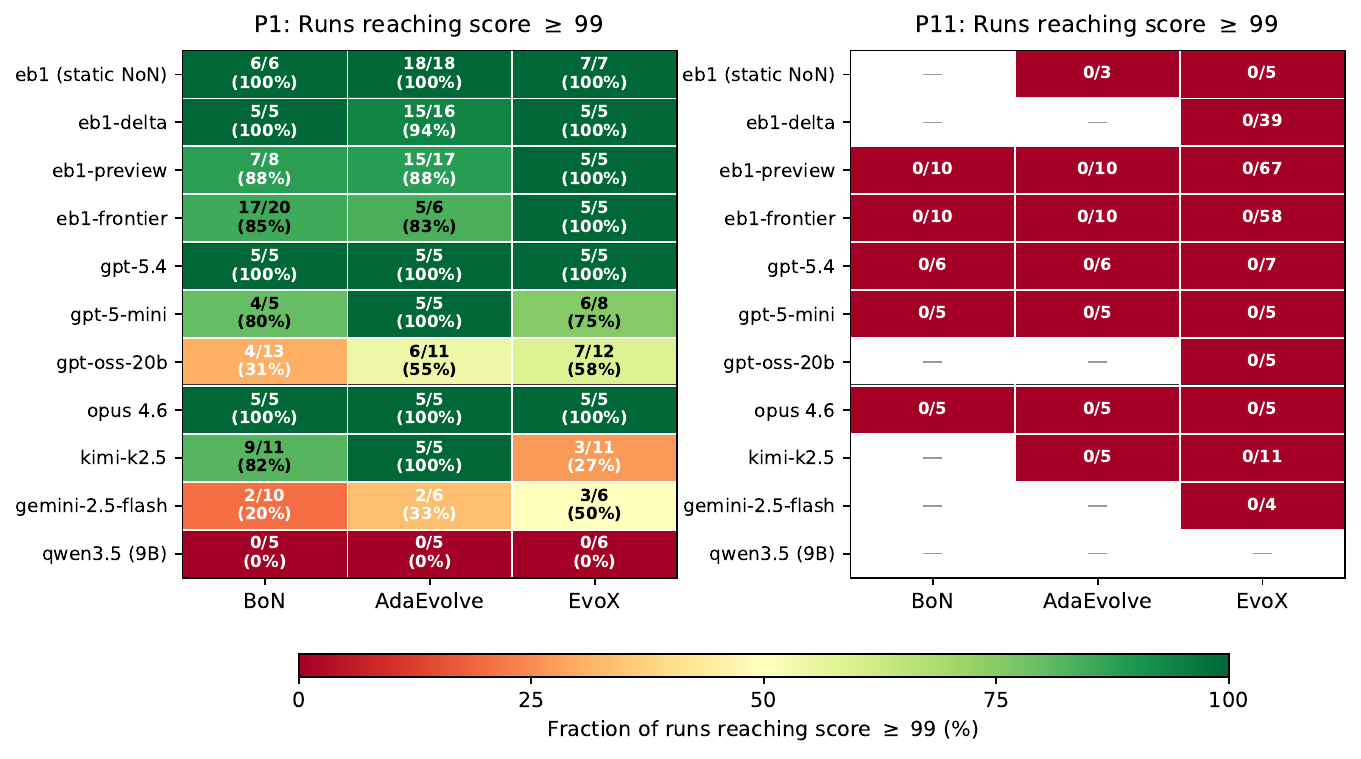}
\caption{$G \times \mathcal{M}$ coverage matrix showing the fraction of runs reaching score~$\geq 99$ on P1 (left) vs.\ P11 (right). Each cell shows runs reaching near-optimal out of total runs. P1 exhibits rich differentiation across generators and mechanisms (Section~\ref{sec:results-p1}); P11 shows universal zero across all 22~tested configurations. Grey cells indicate untested combinations. Generators grouped by family; eb1 variants shown individually (eb1-pro excluded due to availability-related coverage sparsity).}
\label{fig:p1-p11-matrix}
\end{figure*}

\section{Extended comparison with concurrent work}
\label{app:concurrent}

The phenomena we formalize have been observed empirically by several groups working on ADRS during the same period. We provide a detailed comparison here.

\paragraph{Cheng et al.\ (2025).}
\citet{cheng2025letbarbariansin} compare three frameworks (OpenEvolve, GEPA, ShinkaEvolve) across two generators (GPT-5, Gemini-3.0) on ten systems problems and report that ``GEPA performed significantly better with GPT-5, whereas ShinkaEvolve favored Gemini-3.0''---a clear $G \times \mathcal{M}$ interaction across diverse problem domains. Their best practices distill these observations into practitioner heuristics (``select LLM based on desired solution structure,'' ``the solution is only as good as the evaluator'') that correspond to our generator sensitivity and assessor-limited concepts, respectively.

\paragraph{SkyDiscover.}
SkyDiscover \citep{skydiscover2026} extends this to 172 Frontier-CS problems and 14 math/systems tasks, demonstrating both generator sensitivity and mechanism sensitivity at scale, though without run-to-run variance data to assess the significance of individual differences.

\paragraph{EvoX.}
EvoX \citep{liu2026evox} reports mean and best scores over three runs per configuration across $\sim\!200$ tasks, revealing both generator sensitivity (e.g., on Cloudcast, Gemini-3.0-Pro succeeds with random sampling alone while GPT-5 requires an adaptive mechanism---a $G \times \mathcal{M}$ interaction) and mean--best gaps consistent with basin structure. However, EvoX focuses on \emph{within-run} variation: adapting the search strategy across iterations as the landscape changes during a single trajectory. The \emph{across-run} question---why do independent runs of the same configuration produce different outcomes?---is not studied, and the replicated data that could answer it is reported only as a robustness metric.

\paragraph{AdaEvolve.}
AdaEvolve \citep{cemri2026adaevolve} reports mean$\pm$std over three runs per configuration on six math and seven systems tasks, using both GPT-5 and Gemini-3-Pro. Their cross-backbone data reveals $G \times \mathcal{M}$ interactions that the paper itself does not highlight: ShinkaEvolve on Circle Packing scores $2.464 \pm .083$ with GPT-5 but $2.622 \pm .012$ with Gemini, while GEPA on Circle Packing (Rect) \emph{worsens} from $2.326$ to $2.216$ when switching to Gemini. The variance data is also informative: ShinkaEvolve's std of $.083$ on Circle Packing with GPT-5 suggests bimodal outcomes (some runs reaching $2.541$, others not), while AdaEvolve achieves std $= .001$ on the same task. This variance reduction is not accidental---AdaEvolve's Level~3 ``Meta-Guidance'' is explicitly designed to escape stagnation plateaus, which in our terminology are basin commitment events. Their case study confirms this: runs without meta-guidance ``remain stuck near $2.514$,'' while meta-guidance triggers a qualitative strategy shift that escapes to $2.636$. AdaEvolve thus provides engineering validation that basin structure is real and that explicit escape mechanisms can address it, consistent with our theoretical prediction that run-to-run variance reflects basin structure in $L_{\text{eff}}$.

\paragraph{LEVI.}
LEVI \citep{tanveer2026levi} builds a framework around generator--mechanism interaction: cheap models handle most mutations (within-basin refinement in our terminology), while a frontier model is reserved for infrequent ``paradigm shifts'' (cross-basin jumps). LEVI's CVT-MAP-Elites archive, which maintains diversity across both structural and behavioral dimensions, can be understood as a mechanism designed to prevent premature basin commitment. Their controlled comparison---same model, same budget, three seeds---shows that the search architecture matters as much as the generator, consistent with our $G \times \mathcal{M}$ interaction results. LEVI observes run-to-run variance (their shaded confidence bands) but treats it as noise; our framework characterizes its source.

\paragraph{GEPA.}
GEPA \citep{agrawal2025gepa}, a reflective prompt optimizer (Genetic-Pareto), provides evidence that feedback \emph{structure} is a first-class lever. By reflecting on natural-language feedback (execution traces, error diagnostics) rather than the sparse scalar reward used by RL methods such as GRPO, GEPA outperforms GRPO by up to 20\% on individual benchmarks (${\approx}6\%$ on average) with up to $35\times$ fewer rollouts. Consistent with our formalization (Section~\ref{sec:framework}: structured feedback enters through context construction $C(D_t)$, not by extending $\A$'s codomain), this richer-than-scalar feedback reshapes $L_{\text{eff}}$ to be more navigable. GEPA also exhibits $G \times \mathcal{M}$ interaction: their merge mechanism improves performance with GPT-4.1~Mini but \emph{degrades} it with Qwen3~8B.

\paragraph{Cheng et al.\ (2026).}
\citet{cheng2026adrsdatabases} extend ADRS to database optimization and address the $A$-limited regime directly, co-evolving the evaluator alongside solutions. Across three case studies they demonstrate that a misleading $A$ traps evolution in false optima and that improving $A$ unlocks gains no $G$ or $\mathcal{M}$ change could achieve (up to $6.8\times$ latency reduction)---making $A$ adaptive ($L_{\text{eff},t} = A_t \circ G$) even with static $G$.

\paragraph{Glia.}
\citet{hamadanian2026glia} introduce a multi-agent workflow (Researcher + Supervisor) for systems design that reasons at the hypothesis level rather than performing code-level mutation. Compared against EoH, FunSearch, and OpenEvolve on LLM-serving request routing, their Multi-Context variant (MCG, best-of-$N$ parallel reasoning trajectories) outperforms evolutionary methods by $1.3$--$1.7\times$, with diminishing returns beyond $N = 4$. Glia simultaneously changes both generator (compound reasoning agent vs.\ single-model code completion) and mechanism (best-of-$N$ vs.\ island evolution), so the improvement cannot be attributed to either component alone---consistent with performance being a property of the $(G, \A, \mathcal{M})$ triple. Their finding that MCG benefit saturates at small $N$ is consistent with ensemble diminishing returns when $L_{\text{eff}}$ diversity is exhausted. No variance analysis or formal decomposition is provided.

\paragraph{ShinkaEvolve.}
ShinkaEvolve \citep{lange2025shinkaevolve} is an LLM-driven evolutionary program-search framework that maintains a fixed-size archive of evaluated programs with fitness scores, and adds a \emph{meta-scratchpad}: every $T$ generations a meta-agent summarizes recent evaluations into individual program summaries, global insights, and implementation recommendations, appended to the mutation prompt as ``knowledge diffusion.'' In our terminology the meta-scratchpad is distilled context construction $C(D_t)$---a periodically regenerated, lossy textual compression of the run history fed back to $G$---playing the same role as Engram's Research Digest (next) while remaining distinct from the raw program archive. Its strong generator dependence (favoring different backbones than competing frameworks on the same tasks; Cheng et al.\ above) is one of the $G \times \mathcal{M}$ interactions we document.

\paragraph{Engram.}
\citet{karimi2026engram} address what they call the \emph{coherence ceiling}: a single long-running agent's performance degrades as context grows, yet independent runs discard prior insights. Their solution decouples exploration from persistence: a sequence of agents each operates in a fresh context window, reading a structured \emph{Research Digest} summarizing prior agents' findings, and writing results into a persistent \emph{Archive}. On LLM request routing (same benchmark as Glia), multi-cloud multicast, and KV-cache reuse, Engram outperforms Glia, EoH, FunSearch, and OpenEvolve. In our terminology, the coherence ceiling is a practical manifestation of the growing-dimensional state $(D_t, \mathcal{M}_t)$ exceeding what fits in a finite context window; the Digest is a lossy compression of $D_t$ designed to preserve the reasoning behind prior attempts, not just their scores. Their ablation confirms that context construction is a first-class design choice: removing the Digest degrades performance more than removing the raw Archive, suggesting that structured interpretation of history matters more than raw access to it. Engram also tolerates temporary score regression (costs rising from \$772 to \$1104 before recovering to \$644): the Research Digest preserves the reasoning behind the failed attempt, enabling the next agent to persist within a promising algorithmic family rather than retreating to the previous best. Score-only selection would prune the intermediate, making sustained exploration through a temporary regression unlikely without the structured context that history-dependent reasoning provides.

\paragraph{CORAL.}
CORAL \citep{qu2026coral} is a framework for autonomous multi-agent evolution on open-ended problems: long-running agents explore, reflect, and collaborate through a shared persistent memory of attempts, notes, and reusable skills, with asynchronous execution and heartbeat-based interventions (per-iteration reflection, periodic consolidation, and stagnation-triggered redirection). On mathematical and systems-optimization tasks it reports new state-of-the-art results on 8 of 11 problems, and an ablation on three stress-test tasks shows that disabling note and skill creation degrades the final score on each---evidence that knowledge artifacts causally contribute to search quality. In our framework $G$ is a multi-agent autonomous system, the stagnation-triggered redirection is an explicit basin-escape mechanism in $\mathcal{M}$, and the notes/skills repository is collaborative context construction $C(D_t)$ layered over a shared run history. Its distillation of notes and skills on top of the raw attempt log places it toward the curated end of the spectrum below.

\paragraph{Meta-Harness.}
Meta-Harness \citep{metaharness2026} optimizes the task-specific \emph{harness}---the program wrapping a fixed base model---rather than searching for task solutions directly: an agentic proposer is given unrestricted filesystem access to the source code, scores, and execution traces of every previously evaluated harness, maintaining a population and Pareto frontier with no parent-selection rule. It deliberately forgoes curated archives and persistent-memory summaries in favor of raw access to the full history, on the rationale that the proposer improves automatically as coding agents become more capable. In our terms it is the clearest instance of the run history $D_t$ used as the interface directly---the append-only record exposed to $G$, with the outer-loop mechanism held deliberately minimal so that selection reduces to the proposer's free navigation of $D_t$. It marks the raw end of the raw--distilled spectrum for context construction (Table~\ref{tab:memory-spectrum}): where Engram compresses $D_t$ into a Digest, Meta-Harness leaves $D_t$ uncompressed and relocates any distillation into the proposer's in-context reasoning.

\begin{table}[t]
\centering
\small
\caption{How concurrent ADRS systems carry information across iterations, organized along a raw--distilled spectrum of context construction. Each maps to a move within $(G, \A, \mathcal{M})$ or to context construction $C(D_t)$ over the run history $D_t$. ``Distilled?'' indicates whether information is compressed/interpreted before re-entering the generator's context.}
\label{tab:memory-spectrum}
\begin{tabular}{@{}lp{4.0cm}cp{3.4cm}@{}}
\toprule
\textbf{System} & \textbf{Information carried forward} & \textbf{Distilled?} & \textbf{\gamble{} locus} \\
\midrule
FunSearch / AlphaEvolve & program + score archive & no & $\mathcal{M}$ (island / MAP-Elites selection) \\
ShinkaEvolve & archive + meta-scratchpad (summaries, insights) & partial & $C(D_t)$ distillation \\
AdaEvolve & archive + meta-guidance & partial & $\mathcal{M}$ escape $+\, C(D_t)$ \\
Engram & persistent Archive + Research Digest & yes & $C(D_t)$ across fresh contexts \\
CORAL & shared attempts / notes / skills memory & yes & multi-agent $G + C(D_t)$ \\
GEPA & Pareto frontier + reflective NL feedback (traces, diagnostics) & yes (NL) & $C(D_t)$ reflection; rich $\A$ feedback \\
Meta-Harness & full raw history (code + scores + traces) & no (by design) & raw $D_t$ as direct interface \\
\bottomrule
\end{tabular}
\end{table}

\paragraph{Bilevel Autoresearch.}
\citet{qu2026bilevel} meta-optimize the search mechanism itself: an outer loop analyzes the inner autoresearch loop's trace, generates new Python mechanism code (Tabu Search, Multi-Armed Bandit, Orthogonal Exploration), and injects it at runtime. On a GPT pretraining task, mechanism replacement (their Level~2) produces $5\times$ improvement over the inner loop alone, while parameter-level adjustment yields essentially zero gain---an extreme instance of $\mathcal{M}$-limitation where the generator can produce better solutions but the fixed mechanism cannot navigate to them. They observe that the inner loop without intervention exhibits ``near-deterministic repetition'' (the LLM proposes the same changes every iteration), which in our terminology is a barrier in $L_{\text{eff}}$: the generator's prior biases create deterministic trajectories that the mechanism must break. Their high run-to-run variance ($\pm 0.030$, $67\%$ of absolute mean at $n = 3$) further confirms that single runs are unreliable estimators.

\paragraph{Summary.}
These works collectively demonstrate that generator sensitivity, $G \times \mathcal{M}$ interaction, and assessor-dependent performance are robust empirical phenomena spanning competitive programming, math, systems optimization, and database domains. All report variance data but treat it as noise or a robustness metric; none provides a theoretical account of why these phenomena arise or how to diagnose the limiting factor in a given configuration with targeted evaluations rather than exhaustive ablation. The \gamble{} framework takes the first steps to fill this gap.

\end{document}